\documentclass[10pt,twocolumn,letterpaper]{article}

\usepackage[dvipsnames,table]{xcolor}

\usepackage[pagenumbers]{cvpr} 

\usepackage{threeparttable}
\usepackage{times}
\usepackage{epsfig}
\usepackage[english]{babel}
\usepackage{pdfrender}
\usepackage{cite}
\usepackage{balance}
\usepackage{url}
\usepackage{mathtools}
\usepackage{enumerate}
\usepackage{makecell}
\usepackage{comment}
\usepackage{caption}
\usepackage{enumitem}
\usepackage{amsthm}
\usepackage{arydshln}
\usepackage{bm}
\usepackage{siunitx,tabularx,ragged2e,booktabs}
\usepackage{wrapfig}
\usepackage{float}
\usepackage{pifont}
\usepackage{epstopdf}
\usepackage{blindtext}
\usepackage{fancyvrb}
\usepackage{multirow, booktabs}
\usepackage{hhline}
\usepackage{float}

\usepackage{algorithm,algpseudocode}
\usepackage{graphicx}
\usepackage{amsmath}
\usepackage{amssymb}
\usepackage{booktabs}
\usepackage{dpr_fec}
\usepackage{multirow}
\usepackage{adjustbox}
\usepackage[normalem]{ulem}

\newcommand{\algacro}{\textbf{GETA}}
\newcommand{\ppsg}{PPSG}
\newcommand{\optname}{QASSO}
\newcommand{\qadnn}{QADNN}
\newcommand{\resnettwenty}{ResNet20}
\newcommand{\resnetfifty}{ResNet50}
\newcommand{\vggseven}{VGG7}

\newcommand{\bert}{Bert}
\newcommand{\cifarten}{CIFAR10}
\newcommand{\imagenet}{ImageNet}
\newcommand{\squad}{SQuAD}
\newcommand{\cmark}{\ding{51}}%
\newcommand{\xmark}{\ding{55}}%

\newtheorem{proposition}{Proposition}[section]

\usepackage[pagebackref,breaklinks,colorlinks]{hyperref}

\usepackage[capitalize]{cleveref}
\crefname{section}{Sec.}{Secs.}
\Crefname{section}{Section}{Sections}
\Crefname{table}{Table}{Tables}
\crefname{table}{Tab.}{Tabs.}

\fvset{frame=single,framesep=1mm,fontfamily=courier,fontsize=\scriptsize,numbers=left,framerule=.3mm,numbersep=1mm,commandchars=\\\{\}}

\def\cvprPaperID{17169} 
\def\confName{CVPR}
\def\confYear{2025}

\begin{document}

\title{
Automatic Joint Structured Pruning and Quantization for Efficient Neural Network Training and Compression
}

\author{\small
Xiaoyi Qu$^2$, David Aponte$^1$, Colby Banbury$^1$, Daniel P. Robinson$^2$, Tianyu Ding$^1$, Kazuhito Koishida$^1$, Ilya Zharkov$^1$, Tianyi Chen$^{1}$\footnote{Corresponding Author.}\\
\small Microsoft$^1$, Lehigh University$^2$\\
{\tt\small \href{mailto:xiq322@lehigh.edu}{xiq322@lehigh.edu}, \href{mailto:Tianyi.Chen@microsoft.comu}{Tianyi.Chen@microsoft.com}}
}

\maketitle

\begin{abstract}
Structured pruning and quantization are fundamental techniques used to reduce the size of deep neural networks (DNNs), and typically are applied independently.
Applying these techniques jointly via co-optimization has the potential to produce smaller, high quality models. However, existing joint schemes are not widely used because of (1) engineering difficulties (complicated multi-stage processes), 
(2) black-box optimization (extensive hyperparameter tuning to control the overall compression), and (3) insufficient architecture generalization. To address these limitations, we present the framework \algacro{}, which automatically and efficiently performs joint structured pruning and quantization-aware training on any DNNs. \algacro{} introduces three key innovations: (i) a quantization-aware dependency graph (QADG)  that constructs a pruning search space for generic quantization-aware DNN, (ii) 
a partially projected stochastic gradient method that guarantees layerwise bit constraints are satisfied, and (iii) a new joint learning strategy that incorporates interpretable relationships between pruning and quantization. 
We present numerical experiments on both convolutional neural networks  and transformer architectures that show that our approach achieves competitive (often superior) performance compared to existing joint pruning and quantization methods. The source code is available at \url{https://github.com/microsoft/geta}.

\end{abstract}

\section{Introduction}
\label{sec:intro}
Deep neural networks (DNNs) have been widely used in varying applications~\cite{krizhevsky2012imagenet,he2016deep,vaswani2017attention,kenton2019bert}. However, their increasing size has raised several concerns. 
One major challenge is the substantial storage space required to hold these models, which can be impractical for everyday devices such as standard PCs and even more so for resource-constrained edge devices~\cite{shuvo2022efficient}. Furthermore, as model sizes increase, inference cost often lengthens, leading to delays that can be frustrating for users who expect quick responses. 
Therefore, it is of practical interest and importance to compress the model while maintaining performance similar to the full model.  To address the above concerns, various model compression techniques have been studied in recent years~\cite{deng2020model}. 

\begin{figure*}[ht]
\centering
\includegraphics[width=\linewidth]{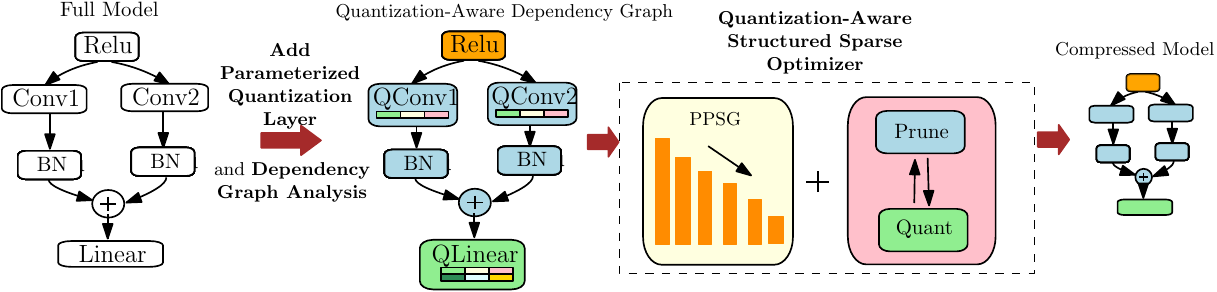}
\caption{\algacro{} framework pipeline. Nodes Conv1 and Conv2 represent two convolutional layers, node BN represents batch normalization, and the ``+" represents summation. For details on the remainder of the figures, see~\cref{sec:quantization}--\cref{sec:algorithm-description}.}
\label{fig:overview.diagram}    
\end{figure*}

Pruning and quantization are two fundamental techniques that are widely deployed, each following different methodologies. Structured pruning is perhaps the most popular pruning scheme, which aims to remove redundant structures in DNNs while preserving performance~\cite{otov1,fang2023depgraph}. Quantization reduces the bit width of the data flowing through a DNN~\cite{deng2020model}. In practice, structured pruning is typically applied first to identify a high-performing subnetwork, which is then quantized to further reduce its size and enhance its processing speed on specified hardware~\cite{han2015deep,louizos2017bayesian}. However, treating pruning and quantization separately has limitations. For example, more heavily structurally pruned models are typically more sensitive to quantization and thus require higher bit widths. Thus, joint structured pruning and quantization becomes an important topic.

\subsection{Challenges}

Many studies~\cite{han2015deep,louizos2017bayesian,yang2020automatic, Clip-q,QST,li2024markov,DJPQ,van2020bayesian, balaskas2024hardware, opq,admm} have combined pruning and quantization to obtain high compression ratios. However, these joint methods are not commonly used in practice due to one or more of the following reasons: engineering difficulty, black-box optimization, and insufficient architecture generalization, which we now discuss.

\smallskip
\noindent
\textbf{Engineering Difficulties.} First, many joint pruning and quantization methods follow a two-stage process. For example, \cite{yang2020automatic, Clip-q, DJPQ, van2020bayesian} first determine the configurations (pruning ratio and bit width) for each layer of the network, and then train the pruned and quantized model. They require separate compression and retraining stages since the two stages may be incompatible with each other. Thus, two-stage pipelines increase the execution time, especially for large datasets (\eg, ImageNet). For these reasons, a one-shot (all-in-once) framework is preferred. Second, while recent automated structured pruning frameworks propose dependency graphs to support generic architectures~\cite{otov2,fang2023depgraph}, integrating quantization introduces new challenges. The addition of attached and inserted branches in the trace graph, which are not accounted for in existing dependency graph analysis, breaks the supports for any architecture.

\smallskip
\noindent
\textbf{Black-box Optimization Process.} 
A significant portion of existing methods lacks explicit control over sparsity and bit width during optimization process. This limitation arises from the multi-objective nature of joint compression problems, which require balancing conflicting goals: maintaining model performance while maximizing sparsity and minimizing bit width. Approaches such as DJPQ~\cite{DJPQ}, BB~\cite{van2020bayesian}, and Clip-Q~\cite{Clip-q} often tackle this challenge by introducing regularization coefficients to reconcile conflicting objectives.
However, a significant drawback is their inability to predict the final compression ratio of the model prior to executing the entire optimization process. Consequently, users typically require extensive hyper-parameter tuning efforts, 
limiting flexibility and productivity in practice.


\smallskip
\noindent
\textbf{Insufficient Architecture Generalization}. The existing work~\cite{yang2020automatic, Clip-q,QST,li2024markov,DJPQ,van2020bayesian, balaskas2024hardware} on joint structured pruning and quantization primarily targets convolutional neural network (CNN), and cannot be applied to architectures such as transformers. 
For instance, both DJPQ~\cite{DJPQ} and BB~\cite{van2020bayesian} applies per-channel pruning to each layer, which will not work for multi-head attention in transformers because it ignores the dependencies between different attention heads.

\begin{table}[h]
\centering
\caption{\algacro{} versus representative joint pruning and quantization methods in terms of \textit{(i)} whether the method supports structured pruning, \textit{(ii)} whether it is a one-shot approach, \textit{(iii)} whether it is a white-box approach, and \textit{(iv)} whether it can be used on a variety of network architectures and tasks (\ie, generalization). Methods not listed lack one or more of these properties.}
\label{tab:method.comparison}
\resizebox{0.47\textwidth}{!}{
\begin{tabular}{l|c|c|c|c|c|c}
\Xhline{2pt}
& \textbf{GETA} & BB & DJPQ & QST & Clip-Q & ANNC \\
\Xhline{0.5pt}
\textbf{Structured Prune}{$^\dagger$} & \cellcolor{yellow!20}{\textcolor{Green}{\cmark}} & \cellcolor{yellow!60}{\textcolor{Green}{\cmark}} & \cellcolor{yellow!60}{\textcolor{Green}{\cmark}} & \textcolor{red}{\xmark} & \textcolor{red}{\xmark} & \textcolor{red}{\xmark} \\
\textbf{One-shot}{$^\dagger$} & \cellcolor{yellow!20}{\textcolor{Green}{\cmark}} & \cellcolor{yellow!60}{\textcolor{red}{\xmark}} & \cellcolor{yellow!60}{\textcolor{red}{\xmark}} & \textcolor{Green}{\cmark} & \textcolor{Green}{\cmark} & \textcolor{red}{\xmark} \\
\textbf{White-box Optimization} & \cellcolor{yellow!20}{\textcolor{Green}{\cmark}} & \cellcolor{yellow!60}{\textcolor{red}{\xmark}} & \cellcolor{yellow!60}{\textcolor{red}{\xmark}} & \textcolor{Green}{\cmark} & \textcolor{red}{\xmark} & \textcolor{Green}{\cmark} \\
\textbf{Generalization} & \cellcolor{yellow!20}{\textcolor{Green}{\cmark}} & \cellcolor{yellow!60}{\textcolor{red}{\xmark}} & \cellcolor{yellow!60}{\textcolor{red}{\xmark}} & \textcolor{red}{\xmark} & \textcolor{red}{\xmark} & \textcolor{red}{\xmark}\\
\Xhline{2pt}
\multicolumn{7}{l}{$^\dagger$ Categorized into engineering difficulties.}
\end{tabular}
}
\end{table}

\subsection{Our Contributions}
\begin{wraptable}{r}{0.48\linewidth}
\vspace{-6mm}
\hspace{-25mm}
\tiny
\begin{Verbatim}[label={\textbf{Framework Usage}}]
\textcolor{magenta}{import} GETA
\textcolor{Green}{\# General DNN model}
geta = GETA(\textcolor{blue}{model})
optimizer = geta.qasso()
\textcolor{Green}{\# Train as normal}
optimizer.step()
\textcolor{Green}{\# Quantized Pruned DNN}
geta.construct_subnet()
\end{Verbatim}
\hspace{-20mm}
\vspace{-7mm}
\end{wraptable}
To tackle the above challenges, we propose \algacro{}, a \textbf{G}eneral and \textbf{E}fficient \textbf{T}raining framework that \textbf{A}utomates joint structured pruning and quantization aware training. By streamlining the workflow, \algacro{} significantly reduces the engineering burdens and minimizes the user intervention (See \textbf{Framework Usage}).

As shown in~\cref{fig:overview.diagram}, \algacro{} begins by incorporating the parameterized quantization layer~\cite{goodparam} into the full model, which allows for layerwise bit widths to be learned during training (see~\cref{sec:quantization}). Next, the framework proposes a quantization-aware dependency graph (QADG) (see~\cref{sec:graph.optimization}) to address previously unconsidered graph transformations introduced by parameterized quantization layers, ensuring support for any architecture. To train the neural network using the quantization-aware dependency graph, we employ a quantization-aware structured sparse optimizer (see~\cref{sec:algorithm-description}) to determine the optimal tradeoff between the pruning ratio and bit width for each layer. Our main contributions are summarized as follows. 


\begin{itemize}[leftmargin=*,noitemsep, topsep=0pt]
\item \textbf{Quantization-Aware Dependency Graph (QADG)}. We propose the quantization-aware dependency graph (QADG) to support joint structured pruning and quantization applied to any quantization-aware deep neural network (QADNN). By eliminating the need to handle each architecture individually, QADG significantly reduces the model-specific engineering workloads. 

\item \textbf{Quantization-Aware Structured Sparse Optimizer (QASSO)}. We propose a quantization-aware structured sparse optimizer, to provide reliable joint structured pruning and mixed precision quantization-aware training.  To the best of our knowledge, QASSO is the first white-box joint optimizer that explicitly controls the sparsity ratio and bit width. Particularly, QASSO employs a partial projected stochastic gradient (PPSG) method to progressively converge towards bit width budget for training stability. Moreover, a joint learning strategy is introduced to address the conflicts between pruning and quantization for performance preservation. 

\item \textbf{Numerical Verification.} We test \algacro{} on a wide range of neural networks including ResNet, VGG, BERT, Phi2, and ViT, among others. The results indicate that \algacro{} achieves competitive (often superior) performance compared to state-of-the-art joint pruning and quantization methods in terms of performance and bit operations.
\end{itemize}

\section{Related Work}

\paragraph{Structured Pruning.} Structured pruning aims to remove redundant structures to reduce the size of DNNs. The identification of redundancies can be performed based on different criteria such as  sparsity~\cite{lin2019toward,wen2016learning,zhuang2020neuron,chen2017reduced,chen2018farsa,chen2021orthant,otov1,gao2020highly,meng2020pruning,yang2020deephoyer,frantar2023sparsegpt,idelbayev2022exploring}, Bayesian pruning~\cite{zhou2019accelerate,van2020bayesian}, ranking importance~\cite{li2020eagleeye,li2019exploiting,zhang2018systematic,otov2}, grouped kernel search~\cite{zhong2023one}, spectral graph analysis~\cite{laenen2023one}, reinforcement learning~\cite{he2018amc,chen2020storage}, and the lottery ticket hypothesis~\cite{frankle2018lottery,frankle2020linear}. 
Previous methods typically use a complicated, time-consuming process that requires extensive domain knowledge to effectively train the DNN. 
Another challenge is to define a pruning search space procedure that can be generalized to various DNNs. 
Recent frameworks, such as OTO~\cite{otov2,otov3,hesso} and DepGraph~\cite{fang2023depgraph}, have automated the construction of this search space using dependency graphs. However, these methods are not suitable for QADNNs due to prevalent issues such as weight-sharing and shape ambiguous operators.\footnote{Shape ambiguous operators are  operators (\eg, \textit{reshape} and \textit{view} in PyTorch) that transform input tensors into outputs of varying dimensions.} This limitation highlights the ongoing challenge of automating structured pruning for any \qadnn{}.


\smallskip
\noindent \textbf{Quantization-Aware Training (QAT).} The standard approach to QAT is applying a  uniform bit width across all layers. However, ~\cite{dong2019hawq,hong2022cadyq} empirically show that different layers in DNNs exhibit different sensitivities to quantization, suggesting that mixed-precision quantization may be a better approach for reducing performance loss. Several strategies including parameterized quantizers~\cite{goodparam}, heuristic approaches~\cite{QST}, reinforcement learning~\cite{elthakeb2020releq,balaskas2024hardware}, multi-objective Bayesian optimization~\cite{miriyala2024mixed}, and Hessian information guided methods~\cite{dong2019hawq,dong2020hawq,yao2021hawq} have been proposed to determine the optimal bit width for each layer.


\smallskip
\noindent \textbf{Joint Pruning and Quantization.} The challenge of using a joint approach lies in determining an optimal tradeoff between the pruning ratio and quantization levels for the model. Two primary strategies have been explored to address this challenge. The first strategy is to efficiently search the joint parameter space with prior work considering heuristics~\cite{QST}, reinforcement learning~\cite{balaskas2024hardware}, and Bayesian optimization~\cite{Clip-q}. The second strategy focuses on gradient-based optimization techniques. \cite{yang2020automatic} formulates a constrained optimization problem and solves it using a combination of ADMM and a greedy algorithm. In the follow up work~\cite{admm}, a reweighted optimization method is proposed with the goal of increasing the compression rate and reducing the number of hyperparameters of the ADMM-based method. \cite{DJPQ} approaches joint pruning and quantization by combining the VIBNet approach~\cite{dai2018compressing} with a differentiable quantizer defined by parameters that are learned.~\cite{van2020bayesian} unifies pruning and quantization by treating pruning as 0-bit quantization. ~\cite{wang2020apq} devises to train a quantization-aware accuracy predictor to deal with large joint parameter search space. 
To avoid a multi-stage process (first determining the configuration and then retraining the model), \cite{opq} proposes a one-shot optimization framework for the joint compression of DNNs. Other strategies are inspired by Markov chain and knowledge distillation. For instance, \cite{li2024markov} presents an interpretable joint pruning and quantization framework that borrows ideas from Markov chain, and \cite{li2024adaptive} applies an adaptive multi-teacher knowledge distillation method to train both the pruned and quantized networks. Our proposed framework falls under the gradient-based optimization approach.

\begin{figure*}[ht]
\centering
\includegraphics[width=\linewidth]{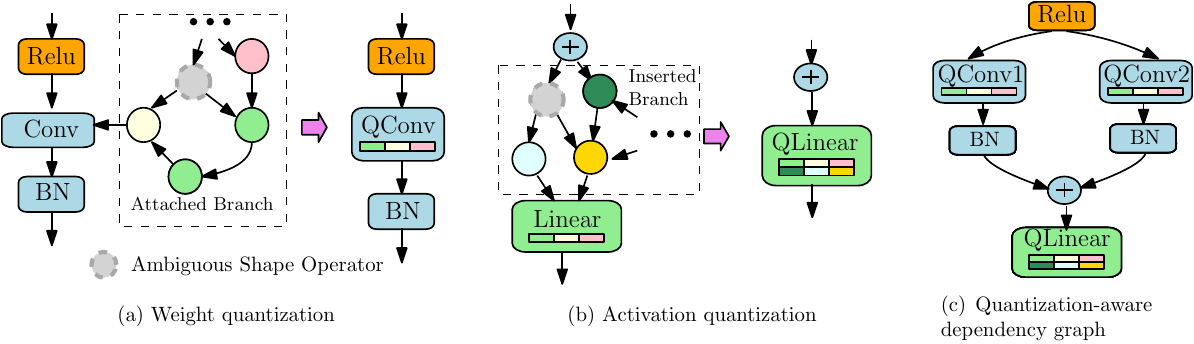}
\caption{Figure \textcolor{red}{2(a)} and \textcolor{red}{2(b)} illustrate the Quantization-Aware dependency graph analysis for weight quantization and activation quantization, respectively. Figure~\textcolor{red}{2(c)} presents a dependency graph after QADG analysis. Concrete examples are provided in~\cref{appendix:dependency.graph}.}
\label{fig:graph.analysis}
\end{figure*}

\section{Quantization with Learnable Parameters}~\label{sec:quantization}
Instead of freezing the bit width in standard QAT approach, we introduce quantization parameters $q_m$, $t$, and $d$ to learn the bit width of each layer~\cite{uhlich2019differentiable}. In particular, $q_m$ represents the maximum value to be mapped and $t$ is the exponent controlling the shape of the  mapping and $d$, known as quantization step size, characterizes the interval between adjacent quantization levels. For each quantization operation, we first quantize the input tensor ${x}$ as $\tilde{{x}}$ by applying a nonlinear function~\cite{DJPQ}
\begin{align}~\label{eq:nonlinear.mapping}
\tilde{{x}} = \text{sgn}({x}) \cdot 
\begin{cases}
|{x}|^t, & |{x}| \leq q_m, \\
(q_m)^t, & |{x}| > q_m.
\end{cases}
\end{align}
After applying the nonlinear mapping, we perform symmetric uniform quantization on $\tilde{{x}}$, resulting in the mapping
\begin{align}~\label{eq:quantization.mapping}
{x}^Q = d\lfloor \tilde{{x}} / d  \rceil,
\end{align}
where $\lfloor \cdot \rceil$ represents rounding to the nearest integer. The associated bit width $b$ is computed as 
\begin{equation}~\label{eq:bit.width.formula}
b = \log_2 \left(\frac{(q_m)^t}{d} + 1\right) + 1.
\end{equation}
To optimize the learnable quantization variables $d, t$, and $q_m$, we compute their gradients using the straight-through estimator\cite{goodparam}. In particular, the gradient of the quantization mapping with respect to $d, t$, and $q_m$ are given by
\begin{align}
&\nabla_d {x}^Q = \text{sgn}({x}) \cdot 
\begin{cases}~\label{eq:gradient.d}
\left( \lfloor \frac{|{x}|^t}{d}\rceil - \frac{|{x}|^t}{d} \right), & |{x}| \leq q_m, \\
\left( \lfloor \frac{(q_m)^t}{d} \rceil - \frac{(q_m)^t}{d} \right), & |{x}| > q_m, \\
\end{cases}\\
&\nabla_t {x}^Q = \text{sgn}({x}) \cdot
\begin{cases}~\label{eq:gradient.t}
|{x}|^t \log (|{x}|), & |{x}| \leq q_m, \\
(q_m)^t \log(q_m), & |{x}| > q_m,
\end{cases}\\
&\nabla_{q_m} {x}^Q = 
\begin{cases}~\label{eq:gradient.qm}
0, & |{x}| \leq q_m, \\
\text{sgn}({x}) t(q_m)^{t-1}, & |{x}| > q_m. \\
\end{cases}
\end{align}
\textbf{Remark.} The computation involving ${x}$ in this section represents element-wise operations.



\section{Quantization-Aware Dependency Graph}~\label{sec:graph.optimization} 
To automate joint structured pruning and quantization-aware training, we first establish a pruning search space. This space is defined as the set of minimally removable structures within the target DNN, ensuring that the remaining sub-network remains functional post-removal. However, establishing this search space automatically is challenging due to the complex topology of DNNs and the diverse roles of operators. Recent advancements in dependency graph~\cite{otov2,fang2023depgraph} address some of these challenges, but existing approaches remain insufficient for \qadnn{}.

To automate the construction of the pruning search space for \qadnn{}, we construct a Quantization-Aware Dependency Graph (QADG). QADG efficiently captures pruning dependencies across both weight and activation quantization. Challenges arise due to the numerous parameterized layers introduced during layer conversion, which include weight-sharing and shape-ambiguous layers that previous algorithms do not account for. Weight and activation quantization-aware layers exhibit distinct structural patterns. As shown in Fig.~\textcolor{red}{\ref{fig:graph.analysis}(a)}, weight quantization introduces a prominent \textit{attached branch} connected to the target layer. In contrast, activation quantization inserts a set of layers between the activation layer and its subsequent layer, referred to as the \textit{inserted branch}.

\begin{algorithm}[ht]
\caption{Constructing a Quantization-Aware Dependency Graph}
\label{alg:algorithm.quantization_aware_dep_graph}
\begin{algorithmic}[1]
\State \textbf{Input:} Trace graph $(\mathcal{V}, \mathcal{E})$ of \qadnn{}.
\State \textbf{Initialize:} $\mathcal{V}_\text{root}^\text{weight}$ = $\emptyset$,  $\mathcal{V}_\text{root}^\text{act}$ = $\emptyset$, and $\mathcal{V}_\text{end}^\text{act}$ = $\emptyset$.
\State Traverse $(\mathcal{V}, \mathcal{E})$ and add the root vertex of each attached branch to the set $\mathcal{V}_\text{root}^\text{weight}$.~\label{line:get_root_vertices_weight_quant}
\For{each $v\in \mathcal{V}_\text{root}^\text{weight}$}\label{line:weight_quant_graph_opt_start} 
    \State Find attached branch associated with root vertex $v$.
    \State Merge vertices in attached branch as vertex $\tilde{v}$.
    \State Replace $v$ with $\tilde{v}$.
\EndFor\label{line:weight_quant_graph_opt_end}
\State Traverse $(\mathcal{V}, \mathcal{E})$ and add the root vertex and end vertex of each inserted branch to $\mathcal{V}_\text{root}^\text{act}$ and $\mathcal{V}_\text{end}^\text{act}$,  respectively. 
\For{each pair $(v_\text{root},v_\text{end}) \in \mathcal{V}_\text{root}^\text{act} \times \mathcal{V}_\text{end}^\text{act}$}
    \State Merge vertices between $v_{\text{root}}$ and $v_{\text{end}}$ as vertex $\tilde{v}$.
    \State Replace $v_{\text{end}}$ with $\tilde{v}$.
    \State Add an edge from $v_\text{root}$ to $\tilde{v}$ into $\mathcal{E}$.~\label{line:add_line_between_root_new_end}
\EndFor
\State Conduct dependency graph analysis in~\cite{otov2}.\label{line:run_dep_graph_otov2}
\State \textbf{Output:} the QADG obtained from~\cref{line:run_dep_graph_otov2}. 
\end{algorithmic}   
\end{algorithm}

\noindent \textbf{Quantization-Aware Dependency Graph Analysis.} To tackle these challenges, as stated in~\cref{alg:algorithm.quantization_aware_dep_graph}, we propose QADG analysis. At~\cref{line:get_root_vertices_weight_quant}, we first traverse the trace graph $(\mathcal{V}, \mathcal{E})$ via depth-first search to identify the set of root vertices, $\mathcal{V}_\text{root}^\text{weight}$, for weight quantization. An example of a root vertex is \textit{Conv} in Fig.~\textcolor{red}{\ref{fig:graph.analysis}(a)}. We then identify attached branches, merge them as new vertices, and replace the root vertices with these new structures, as specified at~\cref{line:weight_quant_graph_opt_start}-\cref{line:weight_quant_graph_opt_end}. For activation quantization, we first locate the root and end vertices, such as \textit{Relu} and \textit{QLinear} in Fig.~\textcolor{red}{\ref{fig:graph.analysis}(a)}. Next, we identify the inserted branches, merge them as new vertices, and replace the end vertices with these new structures. To preserve the connectivity of \qadnn{}, we reconnect the root vertices with the newly formed end vertices in~\cref{line:add_line_between_root_new_end}. Through this optimization, we consolidate complex attached and inserted branches into single entities, allowing us to de-duplicate shared weights and eliminate shape-ambiguous vertices. Subsequently, we apply the dependency graph analysis from~\cite{otov2} to derive the QADG, which facilitates the construction of the pruning search space over the \qadnn{}, enabling joint structured pruning and quantization-aware training.
\section{\optname{}}\label{sec:algorithm-description}
After obtaining a QADG using~\cref{alg:algorithm.quantization_aware_dep_graph}, we obtain the pruning search space of the \qadnn{}, \ie, the parameter groups $\Gcal$. Each $g\in\mathcal{G}$ represents the collection of trainable variables in one minimally removal structure. We then apply our proposed \optname{} optimizer (see Algorithm~\ref{alg:algorithm.JPQ}) to solve the problem
\vspace{-2mm}
\begin{subequations}~\label{prob:main}
\begin{align}
\mathop{\text{minimize}}_{\substack{{{x}} \in \R{n}~\\ (d, q_m, t) \in \R{|\Lcal|} \times \R{|\Lcal|} \times \R{|\Lcal|}} } & f({x}, d, q_m, t)~\label{obj:main} \\
\text{s.t.} \quad \quad \quad \ &  \text{Card}\{g \in \Gcal | [{x}]_g = 0 \} = K,~\label{constr:sparsity} \\
& b_i \in [b_l, b_u], i \in \Lcal,~\label{constr:bit width} 
\end{align}
\end{subequations}
where $K$ represents the target sparsity ratio, $[b_l, b_u]$ specifies the target bit width range, and $\Lcal$ denotes the index set of layers that have parameterized quantization layers added, and $|\Lcal|$ represents the cardinality of set $\Lcal$, and bit width $b_i$ is computed using formula~\cref{eq:bit.width.formula} given in~\cref{sec:quantization}.

\begin{algorithm}[ht!]
\caption{\optname}
\label{alg:algorithm.JPQ}
\begin{algorithmic}[1]
\State \textbf{Inputs:} Initial weight parameters $x$ and quantization parameters $(d,q_m,t)$, learning rate schedule $\{\alpha_l\}$, number of warm-up steps $K_w$, bit width range $[b_l, b_u]$ with $b_u \geq b_l+1$, number of projection periods $B\in[1,b_u-b_l$], bit width reduction factor $b_r\in[1,(b_u-b_l)/B]$, number of projection steps $K_b$, number of pruning steps $K_p$, and number of pruning periods $P$.
\State Perform $K_w$ SGD steps on~\eqref{obj:main} to update $(x,d,q_m,t)$.\label{line:qasso.warmup}
\For{each projection period $ 0,1,\cdots,B-1$}~\label{line:projection.start}
\State $ b_u \gets b_u - b_r$.
\For{$k = 0,1,\cdots,K_b-1$}
\State Update $x$ using one step of SGD on~\eqref{obj:main}.
\State Update $(d,q_m,t)$ using~\cref{alg:ppsg}.
\EndFor
\EndFor~\label{line:projection.end}
\For{each pruning period $ 0,1,\cdots,P-1$}~\label{line:joint.start}
\State Compute saliency score~\cite{hesso} using $x$.
\State Compute the set of important groups $\Gcal_I$ and set of redundant groups $\Gcal_R$ using the saliency score.~\label{line:saliency.score}
\For{$k = 0,1,\cdots,K_p - 1$}
\State Update $(t, q_m)$ using one step of SGD on~\eqref{obj:main}.
\State Stochastic gradient $\hat{\nabla}_x f \approx \nabla_x f(x,d,q_m,t)$.
\State Compute $\gamma$ using \cref{eq:forget.rate.rule}.
\State Update $d$ with \cref{eq:quant.step.size.rule}.~\label{line:d.update}
\State Compute $x^Q$ from~\cref{eq:quantization.mapping}.
\State For currently scheduled learning rate $\alpha$, set
\begin{align}
[x]_{\Gcal_I} &\gets [x]_{\Gcal_I} - \alpha [\hat{\nabla}_x f]_{\Gcal_{I}} \ \ \text{and} \label{line:update.important} \\
[x]_{\Gcal_R} &\gets [x]_{\Gcal_R} - \alpha [\hat{\nabla}_x f]_{\Gcal_R} - \gamma [x^Q]_{\Gcal_R}.\label{line:update.redundant}
\end{align}
\EndFor
\EndFor~\label{line:joint.end}
\State Fixing the quantization parameters, say $(d^*,q_m^*,t^*)$, computed above, train~\cref{obj:main} over the weight parameters in the set of important groups $\Gcal_I$ to obtain $x^*$.\label{step:cooldown} 
\State \textbf{Outputs:} Parameters $(x^*,d^*,q_m^*,t^*)$. 
\end{algorithmic}
\end{algorithm}

\noindent\textbf{Overview of~\optname{}.} Our framework \optname{} (see~\cref{alg:algorithm.JPQ}) aims to compress the size of the DNN while preserving full model performance by removing redundant structures, determining the optimal bit width for each layer that has a parameterized quantization layer added, and recovering knowledge lost during pruning and quantization phases. This is accomplished through a sequential four-stage optimization process: warm-up stage, projection stage, joint stage, and a cool-down stage. The warm-up stage consists of optimizing over all trainable variables using the stochastic gradient 
(SGD) method or any of its variants at~\cref{line:qasso.warmup}, which allows us to achieve a better initialization for improved performance. Next, we step into the projection stage (see \cref{line:projection.start}-\ref{line:projection.end}), where we progressively reduce the bit width range until the bit width constraint~\cref{constr:bit width} is satisfied. This progressive technique enables us to transfer information lost in the low bit precision representation back to the current model. We then proceed to the joint stage (see \cref{line:joint.start}-\ref{line:joint.end}), where we progressively forget the quantized information (see~\cref{line:update.redundant}) within the redundant groups until the constraint~\cref{constr:sparsity} is satisfied. In addition, the bit width selected depends on the amount of information removed within each layer at each step. Specifically, when a significant amount of information is removed, we will consider employing a high bit width for quantization. Once we complete pruning and determine the bit width for each layer, we train the pruned and quantized model until convergence, referred to as the cool-down stage. The projection stage and joint stage are two essential components in our approach and we will discuss them in the next two subsections. 

\subsection{Projection Stage}
During the projection stage, we aim to compute a feasible bit width.  To do so, we consider the problem
\begin{subequations}~\label{prob:projection}
\begin{align}
\min_{\substack{{x} \in \R{n}~\\ (d, q_m, t) \in \R{|\Lcal|} \times \R{|\Lcal|} \times \R{|\Lcal|}} } \ & f({x}, d, q_m, t)~\label{obj:projection} \\
\text{s.t.} \quad \quad \quad \quad & b_i \in [b_l, b_u], \  i \in \Lcal.~\label{constr:projection.bit.width} 
\end{align}
\end{subequations}

\noindent\textbf{Related Approaches and Limitations.} In numerical optimization, projection methods and penalty methods are two of the most common approaches for training DNNs with explicit constraints. However, both approaches are inappropriate for our problem setting~\cref{prob:projection}. On one hand, the projection method is effective when the projection operator has a closed-form solution, while the projection operator associated with~\cref{constr:projection.bit.width} lacks this property. On the other hand, penalty methods (\eg,~\cite{nocedal1999numerical, bertsekas1997nonlinear}) consider a sequence of subproblems that relax the constraint by penalizing its violation in the objective function. Its effectiveness is highly dependent on an appropriate selection of the penalty parameter, which often necessitates hyperparameter tuning.

\begin{algorithm}
\caption{Partial Projected Stochastic Gradient.}
\label{alg:ppsg}
\begin{algorithmic}[1]
\State \textbf{Inputs:} Variables $d$, $q_m$, $t$, and bit width range $[b_l, b_u]$.
\State Update variables $d,q_m,t$ using SGD or its variants.
\State Determine the range $[d_{\min},d_{\max}]$ of $d$ using $(q_m,t)$ and formula~\cref{eq:bit.width.formula}.
\State Project $d$ onto $[d_{\min},d_{\max}]$.
\State \textbf{Outputs:} $d$, $q_m$, $t$.
\end{algorithmic}   
\end{algorithm}

Given the above discussion, we propose a variant of a projected stochastic gradient method called partial projected stochastic gradient (\ppsg) (see~\cref{alg:ppsg}). In this approach, the projection is applied only to the variable $d$. Alternatively, one could apply the projection operation to either $q_m$ or $t$, but our numerical testing shows this often leads to training instability (gradients explode or vanish). This instability stems from exponential transformations in their gradients. The terms $(q_m)^t$ and $(q_m)^{t-1}$ in~\cref{eq:gradient.t}-\eqref{eq:gradient.qm} create highly nonlinear dependencies. Abrupt projection could cause significant value changes, leading to gradient explosions and training collapse. In contrast, the gradient of $d$ is independent of such exponential effects on $d$, making it an ideal candidate to control the bit width range.

\subsection{Joint Stage}
During the joint stage, we aim to identify redundant groups of $\Gcal$, to forget the information within the redundant groups and transfer to the important groups being aware of the quantization parameters, and to determine the layerwise bit widths in terms of the information removed at each layer. 

We first partition our parameter group $\Gcal$ into a set of important groups $\Gcal_I$ and a set of redundant groups $\Gcal_R$ based on saliency scores detailed in~\cite{hesso} at~\cref{line:saliency.score}. 
For variables in $\Gcal_I$, we proceed with vanilla stochastic gradient or its variants at~\cref{line:update.important}. For variables in $\Gcal_R$, we progressively project them to zero by forgetting redundant information at~\cref{line:update.redundant}. Due to the addition of parameterized quantization layers to the original model, weight parameters ${x}$ are converted to its quantized counterpart, denoted as ${x}^Q$. This observation underscores the necessity to forget the quantized information $[{x}^Q]_{\Gcal_R}$ instead of the original information $[{x}]_{\Gcal_R}$. Additionally, it is essential to develop a new update rule for the forget rate $\gamma$ that is aware of quantization parameters to better maintain and transfer the knowledge. 

For ease of notation, we denote the stochastic gradient of function $f({x},d,q_m,t)$ with respect to ${x}$ as $\hat{\nabla}_{{x}} f$. Consequently, the search direction $s({x})$ for updating ${x}$ is
\begin{align}~\label{eq:search.direction}
s({x}) = 
\begin{cases}
- \alpha [\hat{\nabla}_{{x}} f]_g,& g \in \Gcal_I, \\
- \alpha [\hat{\nabla}_{{x}} f]_g - \gamma [{x}^Q]_g,& g \in \Gcal_R.
\end{cases}
\end{align}
The quantized value ${x}^Q$ in~\cref{eq:search.direction} can be rewritten as 
\begin{equation}~\label{eq:quantized.value}
    {x}^Q = \text{sgn}({x}) \cdot \text{clip}_{q_m}^t(|{x}|) + d \cdot \text{sgn}({x}) \cdot R({x}), 
\end{equation}
where the clipped value can be written as
\begin{align}~\label{eq:clip.value}
\text{clip}_{q_m}^t(|{x}|) = 
\begin{cases}
|{x}|^t, & |{x}| \leq q_m, \\
(q_m)^t, & |{x}| > q_m,
\end{cases}
\end{align}
and the residual value is given by
\begin{align}~\label{eq:residual.value}
R({x}) =
\begin{cases}
\lfloor \frac{|{x}|^t}{d}\rceil - \frac{|{x}|^t}{d}, & |{x}| \leq q_m, \\
\lfloor \frac{(q_m)^t}{d}\rceil - \frac{(q_m)^t}{d}, & |{x}| > q_m.
\end{cases}
\end{align}
We denote the angle between $-[\hat{\nabla}_x f]_g$ and $-[\text{sgn}(x)\cdot \text{clip}_{q_m}^t(|x|)]_g$ as $\theta_{\gamma}$ and the angle between $-[\hat{\nabla}_x f]_g$ and $-[\text{sgn}(x)\cdot d \cdot R(x)]_g$ as $\theta_{d}$. The $clip$ represents the mean of the clipped value within the redundant group $\Gcal_R$, i.e., 
\begin{align}~\label{eq:clip.mean}
clip = \mathrm{mean}\left([\text{clip}_{q_m}^{t}(|x|)]_{\Gcal_R}\right).
\end{align}
With the above notations, the forget rate $\gamma$ selection rule is expressed, for pre-specified small $\epsilon$ and $\eta \in (0,1)$, as
\begin{align}\label{eq:forget.rate.rule}
\gamma = 
\begin{cases}
0 , & clip \leq \epsilon, \\
1 - \frac{K_p - k -1}{K_p - k}, & \cos(\theta_{\gamma}) \geq 0, clip > \epsilon, \\
-\frac{(1-\eta)\alpha \|[\hat{\nabla}_x f]_g\|}{\cos(\theta_{\gamma}) \|[\text{sgn}(x) \cdot \text{clip}_{q_m}^{t}(|x|)]_g\|}, &
    \cos(\theta_{\gamma}) < 0, clip > \epsilon.
\end{cases}
\end{align}
The quantization step size $d$ selection rule is, for $\xi \in (0,1)$, 
\begin{align}\label{eq:quant.step.size.rule}
d = 
\begin{cases}
\frac{(q_m)^t}{2^{b_l - 1} - 1}, & \cos(\theta_d) \geq 0, \\
-\frac{\xi \eta \alpha \|[\hat{\nabla}_x f]_g\|}{\gamma \cos(\theta_d)\|[\text{sgn}(x) \cdot R(x)]_g\|}, & \cos(\theta_d) < 0. \\
\end{cases}
\end{align}

\noindent \textbf{Interpretation of Update Rules for $\gamma$ and $d$.} At a high level, the update rule for the forget rate and quantization step size ensures that the search direction in~\cref{eq:search.direction} is a descent direction for the objective function $f$, as stated in~\cref{prop:descent.direction}. Consequently, forgetting knowledge stored in the redundant groups for pruning and quantizing the variables jointly in this manner can make progress towards convergence. Therefore, the conflict between pruning and quantization is largely resolved via our design.

\smallskip
\noindent{}\textbf{Remarks.} When the mean of the clipped values within the redundant group $\Gcal_R$ is relatively small, we reasonably infer that little knowledge is retrieved in the redundant group. Therefore, we set the forget rate to 0 and directly project all parameters in the redundant group $\Gcal_R$ to zero. Otherwise, our forget rate rule is divided based on the angle between the gradient and clipped values. When $\cos (\theta_{\gamma}) \geq 0$, any positive values can be chosen where we select it as a uniform forgetting rate within $K_p$ steps.
The quantization step size $d$ is divided into two cases in terms of the angle between the gradient and the residual values. When $\cos (\theta_d) \geq 0$, $d$ can be selected as any positive values. In this scenario, we consider a low bit width for quantization and specifically, $d$ is selected such that the computed bit width is equal to $b_l$, the min of the bit width range $[b_l, b_u]$. For details of the joint stage implementation, one can refer to~\cref{appendix:safeguard}.

\begin{proposition}~\label{prop:descent.direction}
Let $\hat{\nabla}_x f$ be the full gradient of function $f(x,d,q_m,t)$ with respect to $x$. With forget rate $\gamma$ selection rule~\cref{eq:forget.rate.rule} and quantization step size $d$ selection rule~\cref{eq:quant.step.size.rule}, the search direction $s(x)$ is a descent direction for the function $f$ with respect to $x$ at $x$.
\end{proposition}
\begin{proof}
See \cref{appendix:proof}
\end{proof}


\section{Numerical Experiments}
In this section, we present numerical experiments to demonstrate the effectiveness of our approach, accompanied by ablation studies to assess the contribution of each component to the success of~\algacro{}.~\footnote{Experiment setup details are provided in~\cref{appendix:numerical.experiment.setup}.}

\smallskip
\noindent \textbf{DNN Architectures and Datasets.} The experiments are performed across a wide range of popular CNN architectures, such as VGG7~\cite{karen2014very}, \resnettwenty{}, \resnetfifty{} and RseNet56~\cite{he2016deep}, and transformer architectures, such as Bert~\cite{kenton2019bert}, varying vision transformers~\cite{alexey2020image} and Large Language Models (LLMs) such as Phi2-2.7B~\cite{microsoft2023phi2}. The selected datasets include the benchmark CIFAR10~\cite{krizhevsky2009learning}, ImageNet2012~\cite{deng2009imagenet}, Squad~\cite{kenton2019bert}, and commen-sense benchmarks in LM-Evaluation-Harness~\cite{eval-harness}.

\smallskip
\noindent\textbf{Comparing Methods.} To validate the effectiveness and superiority of our framework, we consider the following methods for comparisons: ANNC~\cite{yang2020automatic}, QST-B~\cite{QST}, DJPQ~\cite{DJPQ} along with its variants, BB~\cite{van2020bayesian}, Clip-Q~\cite{Clip-q}, OBC~\cite{frantar2022optimal}, and a standard first-prune-then-quantize method. All the compared methods consider both pruning and quantization. Furthermore, they use the same strategy that first conducts a search based on the same pretrained model and then fine-tunes the resulting model with the configurations obtained from the search. 

\smallskip
\noindent\textbf{Evaluation Metrics.} We evaluate the performance of each method on two folds, model performance and computational efficiency. The performance depends on the downstream applications with common metrics such as accuracy for image classification and EM or F1-scores for question and answering tasks. Computational efficiency is assessed by BOPs, where lower values indicate more compact models with typically faster inference. For ease of comparison, we report the relative BOP ratio against the baseline full precision models. 



\subsection{CNN Architectures}

\begin{table}[ht]
\centering
\caption{ResNet20 on CIFAR10 dataset. The \textcolor{red}{red} and \textcolor{orange}{orange} represent the best and second-best results, respectively, in the last two columns. Same rule is followed in~\cref{tab:vgg7.cifar10} and~\cref{tab:resnet50.imagenet}.}
\vspace{-3mm}
\label{tab:resnet20.cifar10}
\resizebox{0.475\textwidth}{!}{
\begin{tabular}{lccccc}
\Xhline{2pt}
Method & Pruning & \makecell{Wt~\\ Quant} & \makecell{Act~\\ Quant} & \makecell{Accuracy~\\ (\%)} & \makecell{Rel.~\\ BOPs (\%)} \\
\Xhline{0.5pt}
Baseline & \xmark & \xmark & \xmark  & 91.70 & 100 \\
\hdashline
ANNC~\cite{yang2020automatic} & \text{Unstructured} & \cmark & \xmark  & 90.90 & 6.1 \\
QST-B~\cite{QST} & \text{Unstructured} & \cmark & \xmark  & \cellcolor{red!40} 91.50 & \cellcolor{orange!40} 5.1 \\
\hdashline
\algacro{} & Structured & \cmark & \xmark & \cellcolor{orange!40} 91.42 & \cellcolor{red!40} 4.5 \\
\Xhline{2pt}
\end{tabular}
}
\end{table}

\begin{table*}[ht]
\centering
\caption{Comparison of \algacro{} vs. Structured Pruning followed by Post-Training Quantization (PTQ) for BERT on SQuAD.}
\vspace{-3mm}
\label{tab:joint_vs_sequential_compression_bert_squad}
\begin{minipage}{0.6\linewidth}
\resizebox{\linewidth}{!}{
\begin{tabular}{@{}l@{\hspace{4pt}}c@{\hspace{4pt}}c@{\hspace{4pt}}c@{\hspace{4pt}}c@{\hspace{4pt}}c@{}}
\Xhline{2pt}
Method & Sparsity & EM (\%) & F1 (\%) & \makecell{BOPs\\(GB)} & \makecell{Rel.\\BOPs (\%)} \\
\Xhline{0.5pt}
Baseline & 0\% & 80.08 & 88.50 & 13.57  & 100.0 \\
\hdashline
\multirow{5}{*}{\makecell[l]{OTO~\cite{otov1} followed up 8-bit PTQ}} 
& 10\% & 73.87 & 83.43 & 3.17 & 23.4 \\
& 30\% & 72.95 & 83.31 & 2.71 & 20.0 \\
& 50\% & 72.71 & 83.30 & 2.26 & 16.7 \\
& 70\% & 71.24 & 82.57 & 1.80 & 13.3 \\
\hdashline
\multirow{5}{*}{\algacro{}}
& 10\% & 78.26 & 86.06 & 2.63 & 19.4 \\
& 30\% & 77.28 & 85.70 & 2.29 & 16.9 \\
& 50\% & 76.74 & 85.87 & 1.96 & 14.4 \\
& 70\% & 75.88 & 84.74 & 1.62 & 11.9 \\
\Xhline{2pt}
\end{tabular}
}
\end{minipage}
\begin{minipage}{0.28\linewidth}
\includegraphics[width=0.85\textwidth]{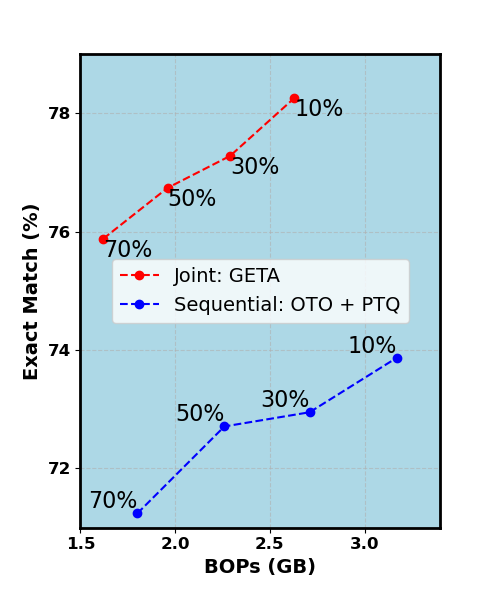}
\end{minipage}
\vspace{-5mm}
\end{table*}

\noindent\textbf{\resnettwenty{} on \cifarten{}.} We first test our framework GETA using ResNet20 on CIFAR10 dataset. For fair comparison, only weight quantization is applied, excluding activation quantization. As shown in~\cref{tab:resnet20.cifar10}, GITA achieves a 4.5\% relative BOPs compression ratio with only a loss of 0.28\% in test accuracy, which demonstrates significantly better performance than ANNC~\cite{yang2020automatic}. Compared to QST-B~\cite{QST}, \algacro{} reduces BOP by 13\% with only a minimal accuracy drop of 0.08\%. We argue that \algacro{} is better suited for practical applications, as QST-B focuses on joint unstructured pruning and quantization. While unstructured pruning tends to deliver higher accuracy at similar compression ratios, its theoretical speedups are challenging to achieve without specialized hardware and software supports~\cite{8916419,zhang2017high,hao2019fpga}. In contrast, the structurally pruned and quantized model produced by \algacro{} is more easily deployed in practical applications.

\begin{table}[ht!]
\centering
\small
\caption{VGG7 on CIFAR10 dataset.}
\label{tab:vgg7.cifar10}
\resizebox{0.47\textwidth}{!}{
\begin{tabular}
{lccccc}
\Xhline{2pt}
Method  & Pruning & \makecell{Wt\\Quant} & \makecell{Act\\ Quant} & \makecell{Accuracy~\\ (\%)} & \makecell{Rel.\\ BOPs (\%)} \\
\Xhline{0.5pt}
Baseline & \xmark & \xmark & \xmark & 93.05 & 100 \\
\hdashline
DJPQ~\cite{DJPQ} & \text{Structured} & \cmark & \cmark & 91.54 & 0.48 \\
DJPQ-restrict~\cite{DJPQ} & \text{Structured} & \cmark & \cmark & 91.43 & 0.46 \\
BB~\cite{van2020bayesian} & \text{Structured} & \cmark & \cmark & \cellcolor{orange!40} 91.96 & \cellcolor{red!40} 0.29 \\
\hdashline
\algacro{} & \text{Structured} & \cmark & \cmark & \cellcolor{red!40} 92.57 & \cellcolor{orange!40} 0.41 \\
\Xhline{2pt}
\end{tabular}
}
\end{table}

\noindent \textbf{\vggseven{} on \cifarten{}.} We then test \algacro{} using \vggseven{} on \cifarten{} to compare with the joint structured pruning and quantization benchmarks. In this case, we enable both weight and activation quantization and report the results in~\cref{tab:vgg7.cifar10}. Based on the results, \algacro{} could significantly outperform other competitors in terms of the test accuracy by 
0.61\% - 1.14\%, and achieves the second best relative BOP ratio which is only worse than BB~\cite{van2020bayesian}. BB separates the model architecture compression and training stages, requiring substantial effort for each. In contrast, \algacro{} offers practical advantages, including efficiency and broad architecture compatibility, enabling an end-to-end, automated joint structured pruning and quantization approach.

\begin{table}[ht!]
\vspace{-2mm}
\centering
\small
\caption{ResNet50 on ImageNet dataset.}\label{tab:resnet50.imagenet}
\vspace{-1mm}
\resizebox{0.475\textwidth}{!}{
\begin{tabular}{lccccccc}
\Xhline{2pt}
Method & Pruning & \makecell{Wt~\\ Quant} & \makecell{Act~\\ Quant}  & \makecell{Accuracy~\\ (\%)} & \makecell{Rel.~\\ BOPs (\%)} \\
\Xhline{0.5pt}
Baseline & \xmark & \xmark & \xmark & 76.13 & 100 \\
\hdashline
OBC~\cite{frantar2022optimal} & Semi-Structured & \cmark & \xmark & 71.47 & 6.67 \\
Clip-Q~\cite{Clip-q} & Unstructured & \cmark & \xmark  & 73.70 & \cellcolor{orange!40} 6.30 \\
\hdashline
\algacro{} (40\% sparsity) & \textbf{Structured} & \cmark & \xmark & \cellcolor{red!40} 75.10 &  6.97  \\
\algacro{} (50\% sparsity) & \textbf{Structured} & \cmark & \xmark & \cellcolor{orange!40} 74.40 &  \cellcolor{red!40} 5.38  \\
\Xhline{2pt}
\end{tabular}
}
\end{table}

\noindent \textbf{\resnetfifty{} on \imagenet{}.} We next test \algacro{} using \resnetfifty{} on \imagenet{}. We select \resnetfifty{} on \imagenet{} because it serves as one of most common benchmarks in structured pruning works~\cite{lin2019towards,fang2023depgraph}, while studies on joint structured pruning and quantization seem absent to the best of our knowledge. Therefore, we compare with joint unstructured pruning or semi-structured pruning and quantization works OBC~\cite{frantar2022optimal} and Clip-Q~\cite{Clip-q}. Unlike the \cifarten{} experiments, we start the training from a pretrained checkpoint. As the results present in~\cref{tab:resnet50.imagenet}, \algacro{} could consistently outperform them in terms of both test accuracy and relative BOP ratios. Considering the difficulty of performance preservation for structured pruning, \algacro{} demonstrates superior performance to existing works.

\subsection{Transformer}
\noindent \textbf{\bert{} on \squad{}.} We now apply \algacro{} to the transformer architecture. The first is the representative encoder-based BERT model~\cite{vaswani2017attention} on the SQuAD benchmark~\cite{rajpurkar2016squad}. While previous works on joint quantization and structured pruning have not been applied to the transformer architecture, we make a more relevant comparison by contrasting our joint optimization approach with the sequential baseline, which first applies pruning-aware training (HESSO)~\cite{hesso} and then performs standard post-training quantization (PTQ)~\cite{torch2019nips}. An alternative sequential baseline, the quantize-then-prune approach, is excluded from our comparison for the following two reasons: \textit{(i)} Applying PTQ to the full model introduces challenges when attempting to prune the model afterward, as calculating gradients with respect to quantized values requires careful handling. \textit{(ii)} A recent work~\cite{harma2024effective} mathematically shows that prune-then-quantize approach is the optimal sequential strategy. Therefore, we focus on comparing \algacro{} with the prune-then-quantize baselines.


The comparison in~\cref{tab:joint_vs_sequential_compression_bert_squad} clearly highlights the advantages of joint structured pruning and quantization during training, versus only pruning at training time and quantization during post-training. At all sparsity ratios, \algacro{} consistently outperforms the multi-stage approach by a large margin. In particular, we observe improvements in exact-match rates (EM) and F1-scores while achieving better compression rates. These results empirically validate that joint pruning and quantization during training is superior to the conventional approach of pruning-aware training followed by post-training quantization, both in terms of model quality and computational efficiency.

\begin{wrapfigure}{r}{0.4\linewidth}
\vspace{-4mm}
\hspace{-4mm}
\begin{minipage}{\linewidth}
\includegraphics[width=\linewidth]{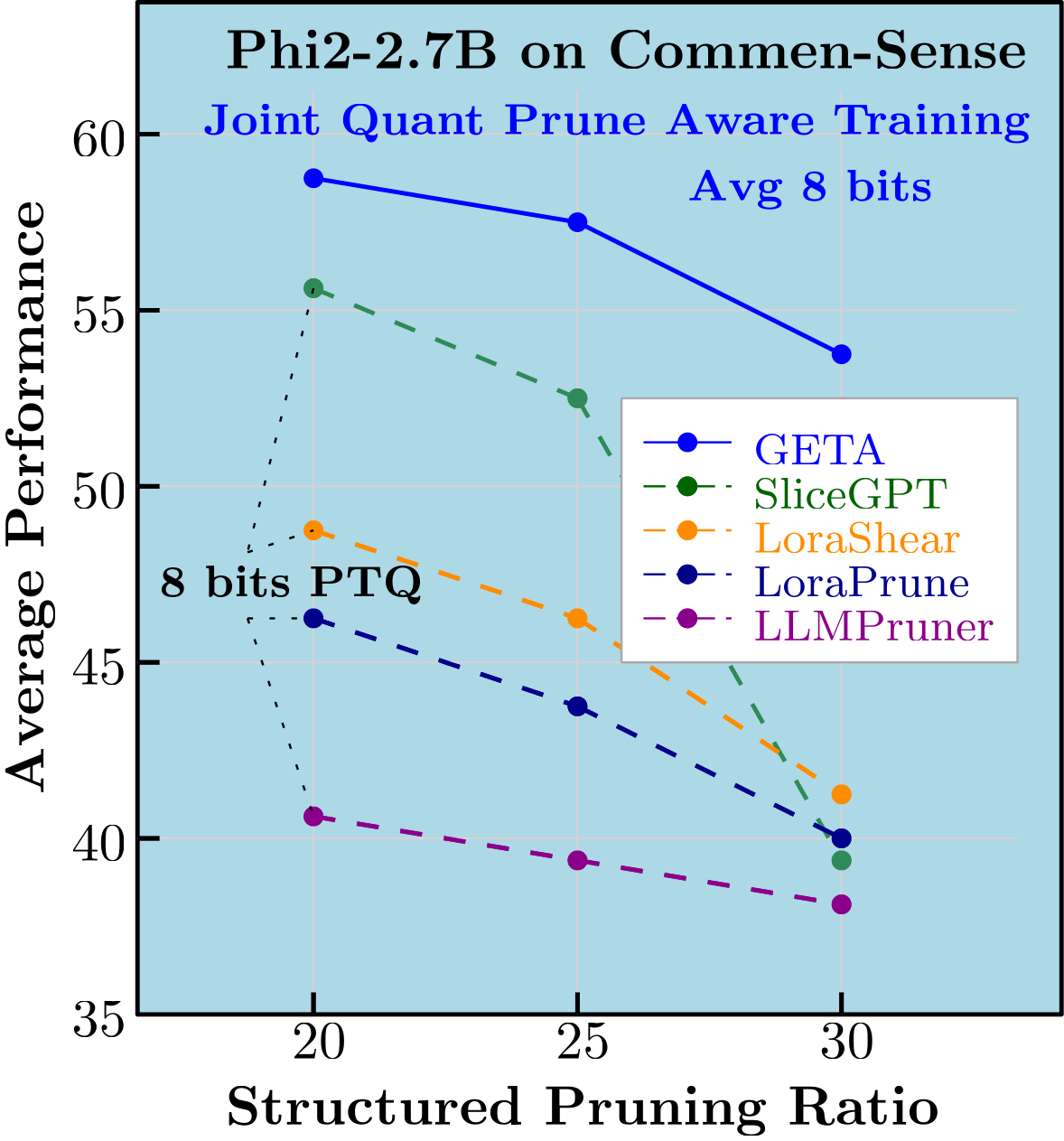}
\end{minipage}
\vspace{-2mm}
\caption{Phi2-2.7B.}
\label{fig:phi2-2.7B}
\vspace{-4mm}
\end{wrapfigure}

\noindent \textbf{Phi2 on Common-Sense.} We next evaluate \algacro{} on popular large language models. Since \algacro{} leverages full gradient information, we select Phi2-2.7B~\cite{microsoft2023phi2}, a model with fewer than 3 billion parameters, to ensure computational feasibility on a single A100 GPU. Similar to the experiments on BERT, we compare \algacro{} with a prune-then-quantize baseline. This baseline first applies pruning-aware training techniques, including SliceGPT~\cite{ashkboos2024slicegpt}, LoraShear~\cite{chen2023lorashear}, LoraPrune~\cite{zhang2023loraprune}, and LLMPruner~\cite{ma2023llm}, followed by PTQ. For a fair comparison, the average bit width across all layers after applying \algacro{} is set to approximately 8 bits, while the baseline uses uniform 8-bit PTQ. As shown in \cref{fig:phi2-2.7B}, \algacro{} consistently outperforms all prune-then-quantize baselines in terms of average performance in common-sense tasks including BoolQ, PIQA, HellaSwag, WinoGrande, ARC-e, ARC-c and OBQA.

\begin{figure*}[ht]
\centering
\begin{subfigure}{0.6\textwidth}
\centering
\resizebox{\linewidth}{!}{
\begin{tabular}
{cccccc}
\Xhline{2pt}
Warmup & Projection & Joint & \makecell{Cool~\\ Down} & \makecell{ResNet56\\(\%)} & \makecell{Phi2\\(\%)} \\
\Xhline{0.5pt}
\cellcolor{yellow!20}{\textcolor{Green}{\cmark}} & \cellcolor{yellow!20}{\textcolor{Green}{\cmark}} & \cellcolor{yellow!20}{\textcolor{Green}{\cmark}} & \cellcolor{yellow!20}{\textcolor{Green}{\cmark}} & \cellcolor{yellow!20}{94.61} & \cellcolor{yellow!20}{58.64} \\
\textcolor{red}{\xmark} & \textcolor{Green}{\cmark} & \textcolor{Green}{\cmark} & \textcolor{Green}{\cmark} & 94.11 &  56.32 \\
\textcolor{Green}{\cmark} & \textcolor{red}{\xmark} & \textcolor{Green}{\cmark} & \textcolor{Green}{\cmark} & 94.10 &  55.17 \\
\textcolor{Green}{\cmark} & \textcolor{Green}{\cmark} & \textcolor{red}{\xmark} & \textcolor{Green}{\cmark} & 93.63 &  52.81 \\
\textcolor{Green}{\cmark} & \textcolor{Green}{\cmark} & \textcolor{Green}{\cmark} & \textcolor{red}{\xmark} & 91.32 &  51.24\\
\Xhline{2pt}
\end{tabular}
}
\caption{}
\label{subfig:ablation.four.stage}
\end{subfigure}
\begin{subfigure}{0.3\textwidth}
\centering
\includegraphics[width=\textwidth]{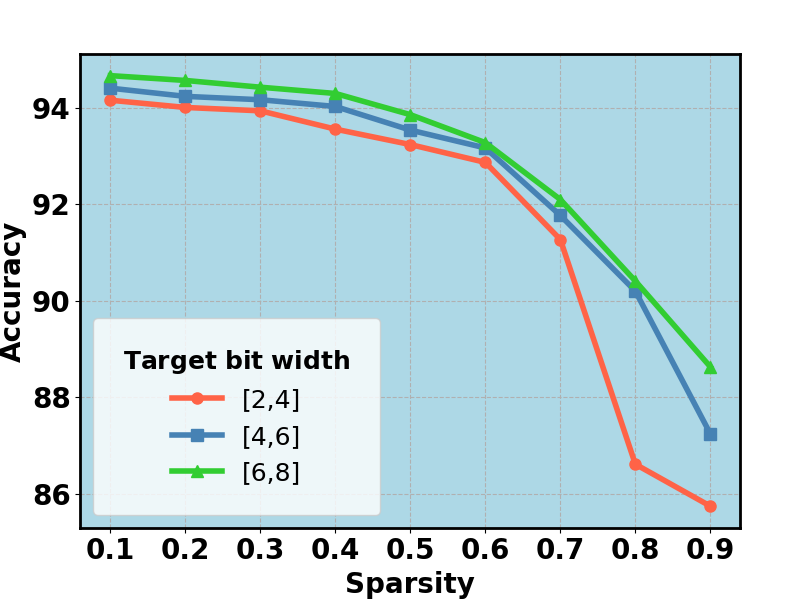}
\caption{}
\label{subfig:ablation.compression.limit}
\end{subfigure}
\caption{The~\cref{subfig:ablation.four.stage} presents an ablation study evaluating the necessity of the four distinct stages of the QASSO optimizer using ResNet56 on the CIFAR10 benchmark and Phi2-2.7B on a common-sense task. The last two columns indicate the model's test accuracy. The~\cref{subfig:ablation.compression.limit} illustrates the limits and boundaries of various compression techniques applied to ResNet56 on the CIFAR10 dataset.}
\label{fig:ablation}
\end{figure*}

\noindent \textbf{Vision Transformers.} Finally, we evaluate GETA on a variety of vision transformer architectures, including SimpleViT~\cite{xie2024jointsq}, ViT~\cite{alexey2020image}, DeiT~\cite{touvron2021training}, Swin Transformer~\cite{liu2021swin}, and Pyramid Vision Transformer~\cite{wang2022pvt}. These models are selected to further validate the architecture-agnostic nature of the \algacro{} framework. To highlight this capability, we focus on reporting the test accuracy and relative BOPs compared to the baseline models. The promising results, as shown in~\cref{tab:vit.architecture}, demonstrate the efficiency and versatility of \algacro{} across diverse transformer architectures.

\begin{table}[ht!]
\centering
\caption{Experiments on various vision transformer architectures.}
\label{tab:vit.architecture}
\resizebox{\linewidth}{!}{
\begin{tabular}
{l|l|cccc}
\Xhline{2pt}
Dataset & Model & Base Acc (\%) & Acc (\%) & Rel. BOPs (\%) \\
\Xhline{0.5pt}
Cifar10 & SimpleViT & 86.48 & \cellcolor{yellow!20} 86.06 & \cellcolor{yellow!20} 4.95 \\
\Xhline{0.5pt}
\multirow{4}{*}{ImageNet} & ViT-Small & 81.43 & \cellcolor{yellow!20} 80.12 & \cellcolor{yellow!20} 19.37 \\
& DeiT-Tiny & 72.01 & \cellcolor{yellow!20} 72.88 & \cellcolor{yellow!20} 16.95 \\
& Swin-Tiny & 80.92 & \cellcolor{yellow!20} 80.09 & \cellcolor{yellow!20} 21.84 \\
& PVTv2-B2 & 81.69 & \cellcolor{yellow!20} 80.53 & \cellcolor{yellow!20} 17.39 \\
\Xhline{2pt}
\end{tabular}}
\vspace{-4.5mm}
\end{table}

\subsection{Ablation Study}~\label{sec:ablation}
Our proposed \optname{} consists of four distinct stages: warm-up stage, projection stage, joint stage, and cool-down stage. To evaluate the contribution of each stage, we conduct an ablation study on two benchmarks, ResNet56 trained from scratch on CIFAR10 and Phi2 fine-tuned from a pre-trained model on the Common-Sense task. The results demonstrate that each stage positively contributes to the model's performance, as measured by test accuracy. As shown in~\cref{subfig:ablation.four.stage}, removing any of the four stages, especially the joint stage and cool-down stage, results in a noticeable decline in test accuracy. The significance of the joint stage and cool-down stage stems from the fact that a significant knowledge transfer is conducted to retain the information lost when applying pruning and quantization.

Moreover, each stage's contribution varies over downstream applications. For instance, the joint stage plays a more critical role when fine-tuning a pre-trained model compared to training from scratch. This can be attributed to the fact that pre-trained models inherently possess a wealth of useful knowledge, and the joint stage helps preserve performance by effectively transferring this knowledge under quantization constraints.


In addition, we perform an ablation study (See~\cref{subfig:ablation.compression.limit}) using ResNet56 on CIFAR10 benchmark to study the limit of each compression technique within GETA framework. As highlighted in~\cite{he2023structured}, structured pruning methods typically achieve sparsity greater than 80\%. However, under joint setup, accuracy begins to degrade significantly beyond 60\% sparsity. This suggests quantization error constrains aggressive pruning, lowering the achievable sparsity threshold from 80\% to 60\% for ResNet56-CIFAR10. For quantization, satisfactory accuracy is typically retained with bit width $\geq 2$bits when sparsity $\leq 60\%$. When sparsity exceeds $60\%$, model becomes less tolerant to lower bit width, requiring at least 4-bit to retain performance.

\section{Conclusion}
We proposed \algacro{}, an automatic framework designed to jointly apply structured pruning and quantization-aware training to deep neural networks, addressing key limitations of existing methods. By leveraging quantization-aware dependency graph analysis, GETA enables structured pruning and quantization-aware training across a wide range of architectures, including both CNNs and transformers. The proposed QASSO optimizer provides explicit control over bit width and sparsity, resolving black-box issues prevalent in existing approaches. With merits such as improved generalization, white-box optimization, and a one-shot framework, GETA offers an easy-to-use and user-friendly solution for practical deployment. In the future, it will be interesting to explore adapting GETA for specialized hardware to improve real-world deployment on different platforms. 

\newpage
{\small
\bibliographystyle{ieee_fullname}
\bibliography{egbib}
}

\newpage
\appendix
\onecolumn
\section{Proof for~\cref{prop:descent.direction}}~\label{appendix:proof}
In this section, we present the proof for~\cref{prop:descent.direction}. For convenience, we restate the proposition as follows.
\begin{proposition}
Let $\hat{\nabla}_x f$ be the full gradient of function $f(x,d,q_m,t)$ with respect to $x$. With forget rate $\gamma$ selection rule~\cref{eq:forget.rate.rule} and quantization step size $d$ selection rule~\cref{eq:quant.step.size.rule}, the search direction $s(x)$ is a descent direction for the function $f$ with respect to $x$ at $x$.
\end{proposition}
\begin{proof}
Denote the full gradient of function $f(x,d,q_m,t)$ with respect to $x$ as $\nabla_x f$. The search direction $s(x)$ is rewritten as
\begin{align}~\label{eq:search.direction.repear}
s({x}) = 
\begin{cases}
- \alpha [\nabla_x f]_g,& g \in \Gcal_I, \\
- \alpha [\nabla_x f]_g - \gamma [{x}^Q]_g,& g \in \Gcal_R.
\end{cases}
\end{align} 
Since $-\alpha[\nabla_x f]_g^T [\nabla_x f]_g < 0$ for $g \in \Gcal_I$, it suffices to show that for $g \in \Gcal_R$,
$$
[\nabla_x f]_g ^T \left[- \alpha [\nabla_x f]_g - \gamma [x^Q]_g\right] < 0.
$$ 
It follows from ~\eqref{eq:quantized.value} that for $g \in \Gcal_R$,
\begin{align*}
-\alpha [\nabla_x f]_g - \gamma [x^Q]_g = \underbrace{- \alpha [\nabla_x f]_g - \gamma  [\text{sgn}(x) \cdot \text{clip}_{q_m}^t(|x|)]_g}_{[s_{\text{clip}}(x)]_g} - \gamma \cdot d [\text{sgn}(x) \cdot R(x)]_g.
\end{align*}
Denote the angle between $-[\nabla_x f]_g$ and $-[\text{sgn}(x)\cdot \text{clip}_{q_m}^t(|x|)]_g$ as $\theta_{\gamma}$. It follows that the vector $[s_{\text{clip}}(x)]_g$ can be decomposed into two orthogonal vectors, i.e.,
$$
[s_{\text{clip}}(x)]_g = [\shat_{\text{clip}}(x)]_g + [\stilde_{\text{clip}}(x)]_g,
$$
where $[\shat_{\text{clip}}(x)]_g$ is orthogonal to vector $[\nabla_x f]_g$ and $[\stilde_{\text{clip}}(x)]_g$ is parallel to vector $[\nabla_x f]_g$. Since $[\shat_{\text{clip}}(x)]_g^T [\nabla_x f]_g = 0$, we have that
$$\|[\shat_{\text{clip}}(x)]_g\| = \gamma \sin\theta_{\gamma}\|[\text{sgn}(x)\cdot \text{clip}_{q_m}^t(|x|)]_g\|.
$$
Using the orthogonality between vector $[\shat_{\text{clip}}(x)]_g$ and vector $[\stilde_{\text{clip}}(x)]_g$, we have that
\begin{align*}
\|[\stilde_{\text{clip}}(x)]_g\|^2 & = \|[s_{\text{clip}}(x)]_g\|^2 - \|[\shat_{\text{clip}}(x)]_g\|^2 \\
& = \|-\alpha [\nabla_x f]_g - \gamma [\text{sgn}(x)\cdot \text{clip}_{q_m}^t(|x|)]_g\|^2 - \gamma^2 \sin^2\theta_{\gamma} \|[\text{sgn}(x)\cdot \text{clip}_{q_m}^t(|x|)]_g\|^2 \\
& = \alpha^2 \|[\nabla_x f]_g\|^2 + 2\alpha \gamma [\nabla f(x)]_g^T [\text{sgn}(x)\cdot \text{clip}_{q_m}^t(|x|)]_g + \gamma^2 \cos^2\theta_{\gamma} \|[\text{sgn}(x)\cdot \text{clip}_{q_m}^t(|x|)]_g\|^2 \\
& = \left[\alpha \|[\nabla_x f]_g\| + \gamma \cos\theta_{\gamma} \|[\text{sgn}(x)\cdot \text{clip}_{q_m}^t(|x|)]_g\|\right]^2. 
\end{align*}
Given the norm and direction of the vector $[\stilde_{\text{clip}}(x)]_g$, we have $[\stilde_{\text{clip}}(x)]_g$ expressed as, for $g \in \Gcal_R$,
\begin{equation}~\label{eq:clip.value.expression}
[\stilde_{\text{clip}}(x)]_g = -\frac{\alpha \|[\nabla_x f]_g\| + \gamma \cos(\theta_{\gamma})\|[\text{sgn}(x)\cdot \text{clip}_{q_m}^t(|x|)]_g\|}{\|[\nabla_x f]_g\|} [\nabla_x f]_g.
\end{equation}
Combining the forget rate selection rule~\eqref{eq:forget.rate.rule} and the expression~\eqref{eq:clip.value.expression} allows us to have that for $g \in \Gcal_R$,  
\begin{equation}~\label{eq:gamma.descent}
\begin{aligned}
[\nabla_x f]_g^T [s_{\text{clip}}(x)]_g & = [\nabla_x f]_g^T \left[[\shat_{\text{clip}}(x)]_g + [\stilde_{\text{clip}}(x)]_g\right]\\
& = [\nabla_x f]_g^T [\stilde_{\text{clip}}(x)]_g \\
& = -\alpha \|[\nabla_x f]_g\|^2 - \gamma \cos(\theta_{\gamma}) \|[\nabla_x f]_g\| \|[\text{sgn}(x) \cdot [\text{clip}_{q_m}^{t}(|x|) ]_g ]\| \\
& < -\eta \alpha \|[\nabla_x f]_g\|^2.
\end{aligned}
\end{equation}
Further, our quantization step size $d$ selection rule~\eqref{eq:quant.step.size.rule} guarantees that
\begin{equation}~\label{eq:d.descent}
-\eta \alpha \|[\nabla_x f]_g\|^2 - \gamma d [\nabla_x f]_g^T [\text{sgn}(x) \cdot R(x)]_g < 0.
\end{equation}
Combining~\cref{eq:gamma.descent} and~\cref{eq:d.descent} allows us to have that
\begin{align*}
    [\nabla_x f]_g^T \left[-[\nabla_x f]_g - \gamma [x^Q]_g\right] & = [\nabla_x f]_g^T \left[[s_{\text{clip}}(x)]_g - \gamma \cdot d [\text{sgn}(x) \cdot R(x)]_g\right] \\
    & < -\eta \alpha \|[\nabla_x f]_g\|^2 - \gamma d [\nabla_x f]_g^T [\text{sgn}(x) \cdot R(x)]_g < 0,
\end{align*}
which completes the proof.
\end{proof}

\section{Joint Stage Implementation Details}~\label{appendix:safeguard}
The update rule in~\cref{eq:quant.step.size.rule} alone is insufficient to ensure that the bit width constraint in~\cref{constr:projection.bit.width} is consistently satisfied. To address this issue, we introduce an algorithm, outlined in~\cref{alg:algorithm.safeguard}, to adaptively adjust the forget rate $\gamma$ and quantization step size $d$ such that the computed bit width stays within the target bit width range $[b_l, b_u]$. Meanwhile, with the adaptive algorithm in place, the search direction $s(x)$ continues to be a descent direction when stochastic gradient $\hat{\nabla}_x f$ is assumed to be full gradient, as demonstrated in~\cref{prop:descent.direction.safeguard}. In addition, there are three hyperparameters that appear in~\cref{eq:forget.rate.rule} and~\cref{eq:quant.step.size.rule} and they are selected as $\eta=0.9$, $\xi=0.999$, and $\epsilon=$1e-8.

\begin{proposition}~\label{prop:descent.direction.safeguard}
Let $\hat{\nabla}_x f$ be the full gradient of function $f(x,d,q_m,t)$ with respect to $x$. With the~\cref{alg:algorithm.safeguard} in place (applied immediately after~\cref{line:d.update} in~\cref{alg:algorithm.quantization_aware_dep_graph}), the search direction $s(x)$ is still a descent direction with respect to function $f$ at the point $x$.
\end{proposition}
\begin{proof}
Denote the full gradient of function $f(x,d,q_m,t)$ with respect to $x$ as $\nabla_x f$. Let's first consider the following three simple cases. When $clip > \epsilon$, the forget rate is equal to 0 and therefore, $s(x)=-\nabla f(x)$, which is a descent direction with respect to function $f(x,d,q_m,t)$ with respect to $x$ at $x$. When $\cos(\theta_{\gamma}) \geq 0$ or $\cos(\theta_d) \geq 0$, $s(x)$ is guaranteed to be descent direction with respect to function $f(x,d,t,q_m)$ with respect to $x$ when $\gamma$ and $d$ are positive values. 

Now, it remains to consider the following two cases: $\cos(\theta_{\gamma}) < 0, clip > \epsilon$ and $\cos(\theta_d) < 0$. As indicated in~\cref{eq:forget.rate.rule} and~\cref{eq:quant.step.size.rule}, we have that $s(x)$ is a descent direction with respect to function $f(x,d,q_m,t)$ with respsect to $x$ if when $\cos(\theta_{\gamma}) < 0$ and $clip > \epsilon$, 
\begin{equation}~\label{eq:gamma.range}
\gamma \in (0, -\frac{\alpha \|[\nabla_x f]_g\|}{\cos(\theta_{\gamma}) \|[\text{sgn}(x) \cdot \text{clip}_{q_m}^{t}(|x|)]_g\|}),
\end{equation}
and when $\cos(\theta_d) < 0$,
\begin{equation}~\label{eq:d.range}
d \in (0, -\frac{\xi \eta \alpha \|[\nabla_x f]_g\|}{\gamma \cos(\theta_d)\|[\text{sgn}(x) \cdot R(x)]_g\|}).
\end{equation}

When $\cos(\theta_{\gamma}) < 0$, we guarantee that with the~\cref{alg:algorithm.safeguard} in place, the forget rate $\gamma$ always lie in the range specified in~\cref{eq:gamma.range} since $\gamma$ either decreases by a factor of $\beta$~(see \cref{line:gamma.decrease}) or remains the same (see~\cref{line:gamma.same}). When $\cos(\theta_d)<0$, we consider two cases based on if the forget rate decreases. If forget rate decreases~(see \cref{line:gamma.decrease}), then the range given in~\cref{eq:quant.step.size.rule} is changed to
\begin{equation}~\label{eq:d.range.new}
(0, -\frac{\xi \eta \alpha \|[\nabla_x f]_g\|}{\beta \gamma \cos(\theta_d)\|[\text{sgn}(x) \cdot R(x)]_g\|})   
\end{equation}
It follows that increasing $d$ by a factor of $\frac{1}{\beta}$ guarantees that $d$ lies within the range~\cref{eq:d.range.new}. If forget rate remains the same, then $d$ always lie in the range~\cref{eq:d.range} since $d$ decreases by a factor of $\beta$. 
\end{proof}

\begin{algorithm}
\caption{Adaptive update rule for $\gamma$ and $d$.}
\label{alg:algorithm.safeguard}
\begin{algorithmic}[1]
\State \textbf{Inputs.} Variables: $\gamma$, $d$, bit width range: $[b_l, b_u]$, $\beta \in (0,1)$, fixed quantization variables: $q_m, t$.
\While{$\log_2 \left(\tfrac{(q_m)^t}{d} + 1\right) + 1 \notin [b_l,b_u]$}
\If{$\log_2 \left(\tfrac{(q_m)^t}{d} + 1\right) + 1 > b_u$}
\State $\gamma \gets \beta \gamma$, $d \gets d/\beta$.~\label{line:gamma.decrease}
\ElsIf{$\log_2 \left(\tfrac{(q_m)^t}{d} + 1\right) + 1 < b_l$}
\State $\gamma \gets \gamma$, $d \gets \beta d$.~\label{line:gamma.same}
\EndIf
\EndWhile
\State \textbf{Outputs.} $\gamma$, $d$. 
\end{algorithmic}   
\end{algorithm}

\section{Numerical Experiment Setup}~\label{appendix:numerical.experiment.setup}
First, we provide details on how we initialize quantization parameters. For each layer that contain quantization parameters, the exponential $t=1$ and the maximum of quantization range $q_m$ is set to the layerwise maximum of the weight tensor. For experiments on ResNet20, VGG7, and ResNet50, the quantization step size $d$ is chosen such that the resulting bit width is 32 bits while for Bert, the quantization step size $d$ is selected to achieve a bit width of 8 bits.

Next, we provide details on how we select the optimizer and the learning rate for different experiments. For ResNet20, we use the SGD optimizer and the initial learning rate is set to 1e-1 with StepLR learning rate scheduler. For experiments of VGG7, we use the optimizer ADAM and the learning rate is set to 1e-3 with StepLR learning rate scheduler. For ResNet50, we use the optimizer SGD and the learning rate is set to 1e-1 with StepLR learning rate scheduler. For Bert, we use AdamW with learning rate as constant 3e-5. For all four experiments, the learning rate for quantization parameters is set as constant 1e-4. For details on how we set hyperparameters related to projection stage and pruning stage, one can find them in~\cref{{tab:experiment.setup}}.
\begin{table}[h]
\caption{Experiment setup for all four experiments. In the following table, the unit for projection steps and pruning steps is the number of epochs. As for Bert, the experiment setups are same under all sparsity ratios (10\%, 30\%, 50\%, 70\%).}
\label{tab:experiment.setup}
\begin{tabular}{lcccccccc}
\Xhline{2pt}
Model & \makecell{Sparsity~\\ level} & \makecell{Total~\\ epochs} & \makecell{Projection~\\ periods $B$} & \makecell{Projection~\\ steps $K_b$} & \makecell{Pruning~\\ periods $P$} & \makecell{Pruning~\\ steps $K_p$} & \makecell{Bit width~\\ reduction $b_r$} & \makecell{Bit width~\\ range $[b_l, b_u]$} \\
\hline
ResNet20 & 0.35 & 350 & 7 & 35 & 5 & 30 & 2 & [4,16]\\
VGG7 & 0.7 & 200 & 5 & 20 & 10 & 30 & 2 & [4,16]  \\
ResNet50 & 0.4,0.5 & 120 & 5 & 5 & 10 & 10 & 2 & [4,16] \\
Bert & 0.1,0.3,0.5,0.7 & 10 & 4 & 1 & 6 & 6 & 2 & [4,16]\\
\Xhline{2pt}
\end{tabular}
\end{table}

\newpage
\section{Quantization-Aware Dependency Graph}~\label{appendix:dependency.graph}
For more intuitive illustration, we present quantization-aware dependency graphs of Bert1 (mini-Bert with one transformer block) and VGG7. Both the original and post-analysis versions of these graphs are shown. To enhance readability of the graph's finer details, we recommend zooming in to a scale of 1500\% or higher using Adobe PDF Reader.
\begin{figure}[ht!]
    \centering
    \includegraphics[width=0.52\linewidth]{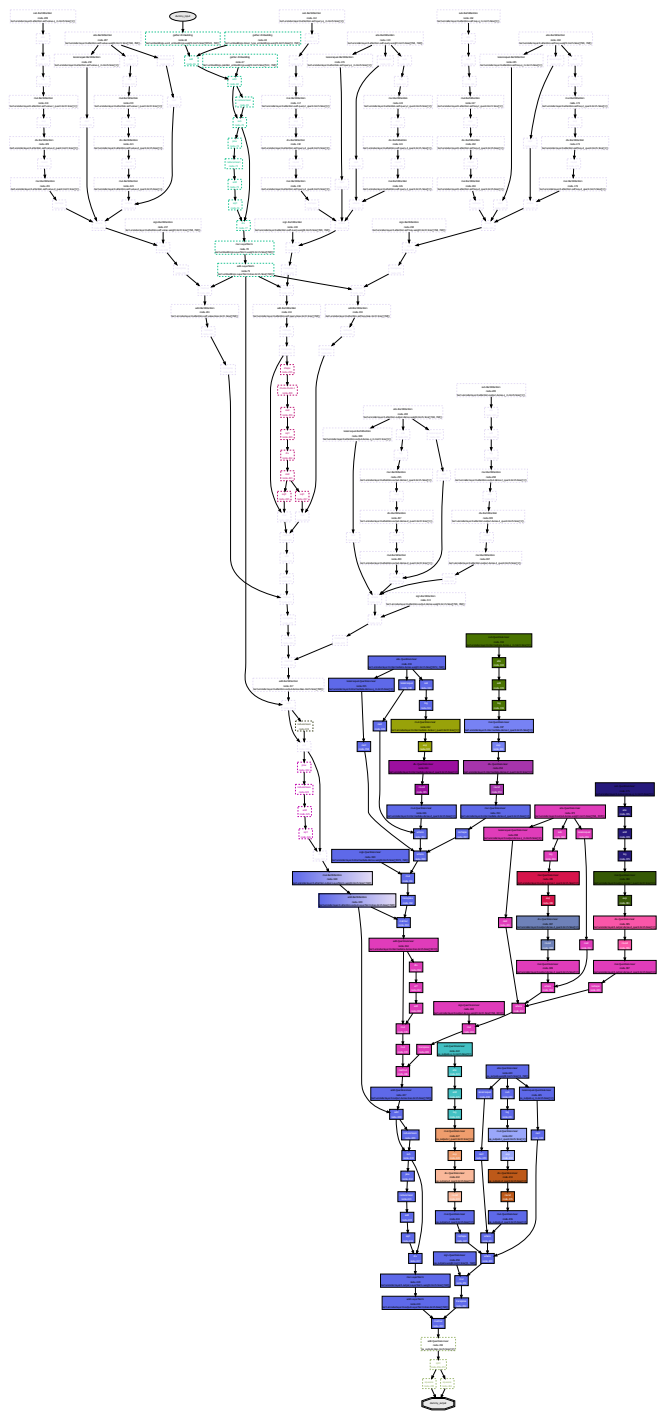}
    \caption{Bert1 before performing quantization-aware dependency graph analysis.}
    \label{fig:Bert.real.example.before}
\end{figure}

BERT.
\begin{figure}[htbp]
    \centering
    \includegraphics[width=0.27\linewidth]{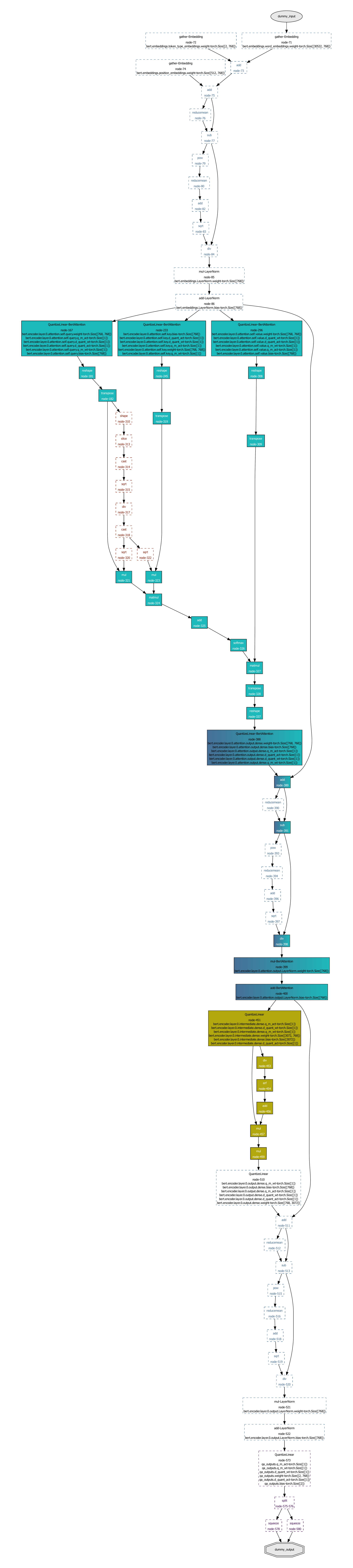}
    \caption{Bert1 after performing quantization-aware dependency graph analysis.}
    \label{fig:Bert.real.example.after}
\end{figure}

VGG7.
\begin{figure}[ht!]
    \centering
    \includegraphics[width=0.80\linewidth]{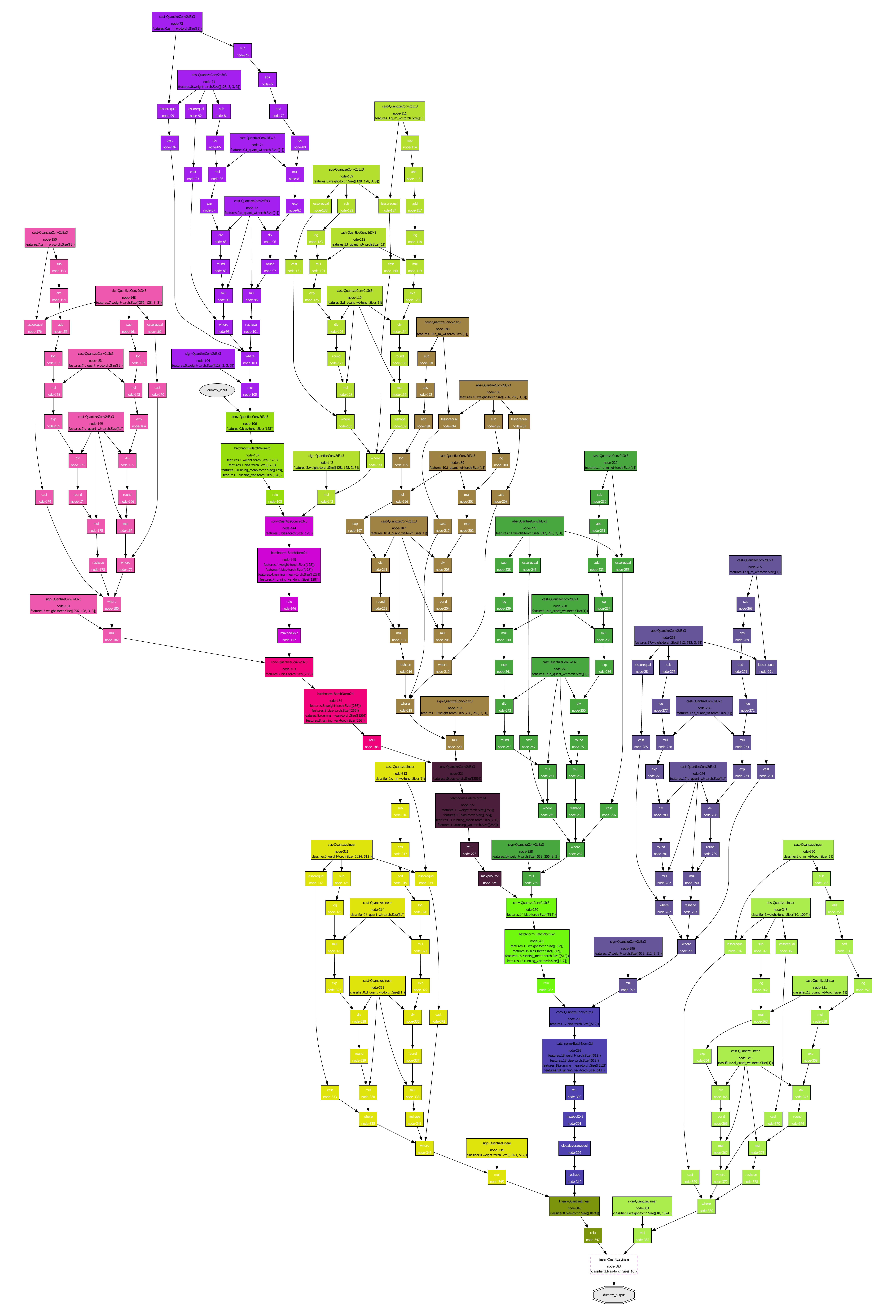}
    \caption{VGG7 after performing quantization-aware dependency graph analysis.}
    \label{fig:vgg7.real.example.before}
\end{figure}

VGG7.
\begin{figure}[ht!]
    \centering
    \includegraphics[width=0.1\linewidth]{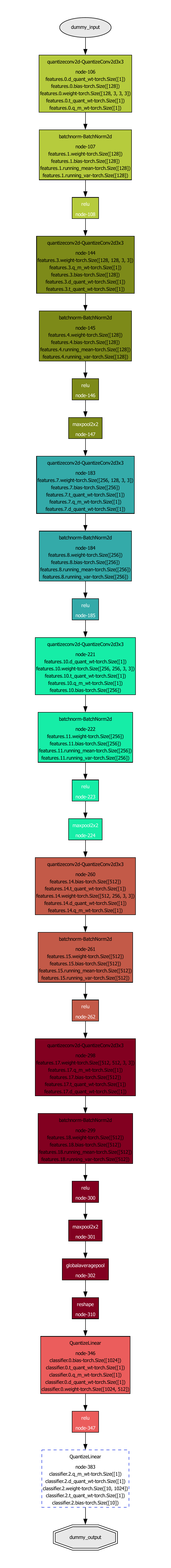}
    \caption{VGG7 after performing quantization-aware dependency graph analysis.}
    \label{fig:vgg7.real.example.after}
\end{figure}

\end{document}


\title{Supplementary Materials}

\author{First Author\\
Institution1\\
Institution1 address\\
{\tt\small firstauthor@i1.org}
\and
Second Author\\
Institution2\\
First line of institution2 address\\
{\tt\small secondauthor@i2.org}
}
\onecolumn

\maketitle
\appendix

\section{Proof and Discussions}
\subsection{Proof of Proposition 5.1}~\label{appendix:prop5.1}
Recall the search direction $s(x_k)$ is defined as
\begin{align}~\label{eq:search.direction}
s(x_k) = 
\begin{cases}
- \alpha [\nabla f(x_k)]_g,& g \in \Gcal_I, \\
- \alpha [\nabla f(x_k)]_g - \gamma_k [x_k^Q]_g,& g \in \Gcal_R.
\end{cases}
\end{align}
\begin{proposition}
With forget rate $\gamma$ selection rule  and quantization step size $d$ selection rule, the search direction $s(x_k)$ is a descent direction with respect to function $f$ at the point $x_k$.
\end{proposition}
\begin{proof}
Due to the fact that negative gradient is guaranteed to be a descent direction, it suffices to show that the search direction $- \alpha [\nabla f(x_k)]_g - \gamma_k [x_k^Q]_g$ is descent direction. 
It follows that~\eqref{eq:search-direction-expression} can be rewritten as
\begin{align*}
s(x_k) = 
\begin{cases}
- \alpha [\nabla f(x_k)]_g, g \in \Gcal_I, \\
- \alpha [\nabla f(x_k)]_g - \gamma_k  [\text{sgn}(x_k) \cdot \text{clip}_{q_m}^t(|x_k|)]_g - \gamma_k \cdot d [\text{sgn}(x_k) \cdot R(x_k)]_g,  g \in \Gcal_R.
\end{cases}
\end{align*}
Denote the angle between $-[\nabla f(x)]_g$ and $-[\text{sgn}(x)\cdot \text{clip}_{q_m}^t(|x|)]_g$ as $\theta_{\gamma}$ and the angle between $-[\nabla f(x)]_g$ and $-[\text{sgn}(x)\cdot d \cdot R(x)]_g$ as $\theta_{d}$. We first decompose the search direction $s_{\text{clip}}(x_k)$ into two orthogonal directions, i.e.,
$$
s_{\text{clip}}(x_k) = \shat_{\text{clip}}(x_k) + \stilde_{\text{clip}}(x_k),
$$
where $\shat_{\text{clip}}(x)^T [\nabla f(x)]_g = 0$ and $\|[\shat_{\text{clip}}(x)]_g\| = \| \gamma [x]_g \cos(\theta_{\gamma} - 90^{\circ})\| = \gamma \sin\theta_{\gamma}\|[x]_g\|$. Given the fact that $\shat_{\text{clip}}(x)$ and $\stilde_{\text{clip}}(x)$ are orthogonal to each other, we have that
\begin{align*}
\|\stilde_{\text{clip}}(x)\|^2 & = \|s_{\text{clip}}(x)\|^2 - \|\shat_{\text{clip}}(x)\|^2 \\
& = \|-\alpha [\nabla f(x)]_g - \gamma [\text{sgn}(x)\cdot \text{clip}_{q_m}^t(|x|)]_g\|^2 \\
& - \gamma^2 \sin^2\theta_{\gamma} \|[\text{sgn}(x)\cdot \text{clip}_{q_m}^t(|x|)]_g\|^2 \\
& = \alpha^2 \|[\nabla f(x)]_g\|^2 + 2\alpha \gamma [\nabla f(x)]_g^T [\text{sgn}(x)\cdot \text{clip}_{q_m}^t(|x|)]_g \\
& + \gamma^2 \cos^2\theta_{\gamma} \|[\text{sgn}(x)\cdot \text{clip}_{q_m}^t(|x|)]_g\|^2 \\
& = \left[\alpha \|[\nabla f(x)]_g\| + \gamma \cos\theta_{\gamma} \|[\text{sgn}(x)\cdot \text{clip}_{q_m}^t(|x|)]_g\|\right]^2. 
\end{align*}
Given the norm and direction of the vector $[\dtilde(x)]_g$, we have $\dtilde(x)$ expressed as
\begin{equation*}
-\frac{\alpha \|[\nabla f(x)]_g\| + \gamma \cos(\theta_{\gamma})\|[\text{sgn}(x)\cdot \text{clip}_{q_m}^t(|x|)]_g\|}{\|[\nabla f(x)]_g\|} [\nabla f(x)]_g.
\end{equation*}

The selection rule~\eqref{eq:forget-rate-rule} ensures that  
\begin{align*}
-\alpha \|[\nabla f(x)]_g\|^2 - \gamma_k [\nabla f(x)]_g^T [\text{sgn}(x) \cdot [\text{clip}_{q_m}^{t}(|x|) ]_g ] < -\eta \alpha \|[\nabla f(x)]_g\|^2.
\end{align*}

The selection rule~\eqref{eq:forget.rate.rule} ensures that  
\begin{align*}
-\alpha \|[\nabla f(x)]_g\|^2 - \gamma_k [\nabla f(x)]_g^T [\text{sgn}(x) \cdot [\text{clip}_{q_m}^{t}(|x|) ]_g ] < -\eta \alpha \|[\nabla f(x)]_g\|^2.
\end{align*}
After that, we can choose quantization step size $d$ based on the info of $\gamma$. Our goal is to guarantee that 
\begin{equation*}
-\eta \alpha \|[\nabla f(x)]_g\|^2 - \gamma d [\nabla f(x_k)]_g^T [\text{sgn}(x_k) \cdot R(x_k)]_g < 0
\end{equation*}

\end{proof}

\subsection{Proof of Proposition 5.2}
\begin{proposition}
$s_k$ is a descent direction with respect to function $f$ at the point $x$.
\end{proposition}

{\small
\bibliographystyle{ieee_fullname}
\bibliography{egbib}
}